\documentclass{svmult}
\usepackage[utf8]{inputenc}
\usepackage{bussproofs}
\usepackage{amsfonts}
\usepackage{amsmath}
\usepackage{amssymb}
\usepackage{tikz}
\usepackage{array}
\usepackage{gb4e}
\usepackage{enumitem}
\usepackage{proof}
\usetikzlibrary{shapes.geometric}
\usepackage[colorlinks,citecolor=red]{hyperref}
\usetikzlibrary{positioning}

\newcommand{\bs}{\backslash}
\newcommand{\bo}{[}
\newcommand{\bc}{]}
\newcommand{\ma}[1]{\textit{#1}}

\definecolor{gray80}{gray}{0.80}

\begin{document}

\title*{Logical Semantics, Dialogical Argumentation, and Textual Entailment} 
\author{Davide Catta,   Richard Moot, Christian Retoré} 
\institute{David Catta, Richard Moot, Christian Retoré \at LIRMM, Univ Montpellier, CNRS, Montpellier, France  e-mail:firstname.lastname@lirmm.fr} 

\maketitle

\abstract*{ 
In this chapter, we introduce a new dialogical system for first order classical logic which is close to natural language argumentation,   and we  prove its completeness with respect to usual classical validity. 
We combine our dialogical system with the Grail syntactic and semantic parser developed by the second author in order to address automated textual entailment, that is, we use it for deciding whether or not a sentence is a consequence of a short text.  This work --- which connects natural language semantics and  argumentation with dialogical logic ---  can be viewed as a step towards an inferentialist view of natural language semantics.} 

\abstract{ 
In this chapter, we introduce a new dialogical system for first order classical logic which is close to natural language argumentation,   and we  prove its completeness with respect to usual classical validity. 
We combine our dialogical system with the Grail syntactic and semantic parser developed by the second author in order to address automated textual entailment, that is, we use it for deciding whether or not a sentence is a consequence of a short text.  This work --- which connects natural language semantics and  argumentation with dialogical logic ---  can be viewed as a step towards an inferentialist view of natural language semantics.} 

\section{Introduction}
We propose an approach to the semantics of natural language inspired by the inferentialist theory of meaning. Inferentialism is a philosophical current developed in the 20th century from a certain reading of Wittgenstein's philosophical research \cite{Wittgenstein1953investigations}. As is well known, the Austrian philosopher declared that the meaning of a linguistic expression is not given  by its truth value but by its \emph{use} in a linguistic context. The inferentialists take up Wittgenstein's theory by claiming that the meaning of a sentence is given by its inferential use within language. That is, the meaning of a sentence lies in its  argumentative use, its justifications, its refutations and more generally its deductive relation to other statements. We formally develop this kind of approach to the theory of meaning using dialogical logic. Dialogical logic is an approach to formal semantics developed from Lorenzen's pioneering work in the 1950s \cite{lorenzen1978dialogische}. In dialogical logic the concept of validity is not defined starting from that of interpretation in a structure. On the contrary, a formula is considered valid if, and only if, there is an argumentative strategy that, if followed, allows those who claim the formula to win any argumentative dialogue that may arise when the formula is proferred.\\
The link with natural language semantics is obtained using type-logical grammars. Such grammars compute the possible meanings of a sentence viewed as logical formulae. In particular the \texttt{Grail} platform, a wide-scale categorial parser which maps French sentences to logical formulas \cite{Moot2017,Moot2015jlm,MootGrail2018}, will be presented. We then use \texttt{Grail} and dialogical logic to solve some examples of textual entailment from the FraCaS database \cite{fracas}.

This chapter is structured as follows. Section~2 introduces inferentialism and its motivations. Section~3 introduces dialogical logic for classical first order logic and dialogical validity. Section~4 gives a proof of the fact that the class of formulas that are dialogicaly valid is  equal to the class of formulas that are valid in the standard meaning of the term i.e.\ true in all interpretations.   Section~5 is an introduction to categorical grammars and \texttt{Grail}. We conclude with the applications of semantics to the problem of textual inference.

We then show applications of this semantics to the problem of textual entailment using examples taken from the FraCaS database. 

Hence this paper deepens and extends our first step in this direction ``inferential semantics and argumentative dialogues''  \cite{CPR2017jla}, mainly because of a strengthened link with natural language semantics and textual entailment: 
\begin{itemize} 
\item 
We do not require the formulas to be in a normal negative form as done in  \cite{CPR2017jla}. In this way the strategies correspond to derivations in a two-sided sequent calculus, i.e.\ with multiple hypothesis and multiple conclusions. 
\item 
We use Grail to connect natural language inference (or textual entailment) to argumentative dialogues: indeed Grail turns natural language sentences into logical formulas (delivered as DRS) --- while our previous work only assumed that sentences could be turned into logical formulas. 
\item This lead us closer to inferentialist semantics: a sentence $S$ can be interpreted as all argumentative dialogues (possibly expressed in natural language) whose conclusion is $S$ --- under assumptions corresponding to word meaning and to the speaker beliefs. 
\end{itemize}

\section{Inferentialism}
A problem with the standard view of both natural language semantics and of logical interpretations of formulas is that the models or possible worlds in which a sentence is true cannot be computed or even enumerated \cite{MootRetoreCompu}.
As far as pure logic is concerned there is an alternative view of meaning called \emph{inferentialism} \cite{Dummett1975-DUMWIA,Dummett1991-DUMTLB,Prawitz2012,Cozzo1994,Brandom2000}. Although initially inferentialism  took place within a constructivist view of logic \cite{Dummett1975-DUMWIA}, there is no necessary conceptual connection between accepting an inferentialist position and rejecting classical logic as explained in \cite{Cozzo1994}. 

As its name suggests \emph{inferentialism}  takes the inferential activity of an agent to be the primary semantic notion rather then truth.   According to this paradigm, the meaning of a sentence is viewed as the \emph{knowledge} needed to understand the sentence.

This view is clearly stated by Cozzo \cite{Cozzo1994}
\begin{quote}
A theory of meaning should be a theory of understanding. The meaning of an expression (word or  sentence) or of an utterance is what a speaker-hearer must know (at least implicitly) about that expression,  or that utterance, in order to understand it.
\end{quote}

This requirement has some deep consequences: as speakers are only able to store a finite amount of data,  the knowledge needed to understand the meaning of the language  itself should also be finite or at least recursively enumerable from a finite set of  data and rules.  Consequently, an inferentialist cannot agree with the Montagovian view  of the meaning of a proposition as the possible worlds in which the proposition is true. In particular because there is no finite way to enumerate the infinity of possible worlds nor to finitely enumerate the infinity of individuals and of relations in a single of those possible worlds. 

Let us present an example of the  knowledge needed
to understand a word. If it is a referring word like \ma{eat}, one should know what it refers to \ma{the action that someone eats something}, possibly some postulates related to this word like \emph{the eater is animated},\footnote{Our system is able to deal with metaphoric use, like the \textit{The cash machine ate my credit card.} see e.g. \cite{Retore2014types2013}.} and how to compose it with other words \ma{eat is a transitive verb, a binary predicate}; if it is a non referring word like \textit{which} one should know that it combines with a  sentence without subject (a unary predicate), and a noun (a unary predicate), and makes a conjunction of those two predicates.  

% Imagine that we want to know what is the meaning of the definite article "la" in italian. According to our description one should ask "what  knowledge should one possess in order to understand the meaning of "la"?". The answeer should be something of the form 
% \begin{description}
% \item[la] One should known that "la" is  a word that when attached to a common noun produce a noun phrase if the common noun is feminine and singular.
% \end{description}

Observe that this knowledge is not required to be explicit for human communication. Most speakers would find difficult to explicitly formulate these rules, especially the grammatical ones. It does not mean that they do not possess this knowledge. 

% -rather simplistic- meaning explanation (1). This does not imply that they do not have the knowledge of the definite article "la". It only shows that knowledge of the meaning could be implicit knowledge. This distinction should not sound problematic. Using a chomskian terminology we could say that implicit knowledge of the meaning correspond to linguistic competence and explicit knowledge corresponds to linguistic performance.
% The aim of a theory of meaning should be then to specify what is linguistic competence i.e., what must a speaker know to be a competent speaker.

An important requirement for a theory of meaning is that the speaker's knowledge can be  \emph{observed}, i.e.\  his knowledge can be observed in the interactions between the speaker(s),  the hearer(s) and the environment. This requirement is supported by the famous  argument against the \emph{private language} of Wittgenstein \cite{Wittgenstein1953investigations} 
This might be explained as follows. 
Imagine that two speakers have the same use of a sentence $S$ in all possible circumstances. Assume that one of the two speakers includes as part of the meaning of $S$  some ingredient  that cannot be observed. This ingredient has to be ignored when defining the knowledge needed to master the meaning of $S$. Indeed, according to the inferentialist view, 
a  misunderstanding that can neither be isolated nor observed should be precluded.

Another requirement  for a theory of meaning is the distinction between sense and force, on which we shall be brief. Since Frege, philosophy of language introduced the distinction between  the sense of the sentence and its force. The sense of a sentence is the propositional content conveyed by the sentence,  while its force is its \emph{mood} --- this use of the word ``mood'' is more general than its linguistic use for verbs.  Observe that the  same propositional content can be asserted, refuted, hoped etc. as  in the three following sentences
\ma{Is the window open?} \ma{Open the window!}  
\ma{The window is open.}. Here we focus on assertions, and questions.

\section{Dialogical Logic}
Usually the inferentialist view of meaning is formally developed, as far as logic is concerned, by the use of natural deduction systems \cite{Francez2015pts}. We deviate from this type of treatment by proposing to formally implement the ideas of inferentialist meaning theory through the tools provided by dialogical logic. 
  In our view, the connection with semantics based on the notion of argument is clearer within the latter paradigm: an argument in favor of a statement is often developed when a critical audience, real or imaginary, doubts the truth, or the plausibility of the proposition. In this case, in order to successfully assert the statement,  a speaker or proponent of it must be capable of providing all the justifications that the audience is entitled to demand. Taking this idea seriously, an approximation of the meaning of a sentence in a given situation can be obtained by studying the \emph{argumentative dialogues} that arise once the sentence is asserted in front of such a critical audience. This type of situation is captured, with a reasonable degree of approximation by dialogical logic. In the dialogical logic framework, knowing the meaning of a sentence means  being able to provide a justification of the sentence in front of a critical audience. Note that with this type of methodology the requirement of manifestability required to attribute knowledge of the meaning of a sentence to a locutor is automatically met. The locutor who asserts a certain formula is obliged to make his knowledge of the meaning manifest so that he can answer the questions and objections of his interlocutor.  In addition, any concessions made by his interlocutor during the argumentative dialogue will form the linguistic context in which to evaluate the initial assertion.

We now gradually present dialogical logic.   
Although the study of dialectics --- the art of correct debate --- and logic --- the science of valid reasoning --- have been intrinsically linked since their beginnings \cite{Castelnerac2013,Novaes,Castelnerac2009},  modern mathematical logic had to wait until the 50s of the last century to ensure that the logical concept of validity was expressed through the use of dialogical concepts and techniques. Inspired by the Philosophical Investigations of Wittgenstein \cite{Wittgenstein1953investigations}, the German mathematician and philosopher Lorenzen \cite{lorenzen1978dialogische} proposed to analyze the concept of validity of a formula $ F $ through the concept of winning strategy in a particular type of two-player game. This type of game is nothing more than an argumentative dialogue between a player, called Proponent, which affirms the validity of a certain formula $ F $ and another player, called opponent, which contends its validity. 
The argumentative dialogue starts by the proponent affirming a certain formula. The opponent takes turns and attacks the claim made by the proponent according to its logical form. The proponent can, depending on his previous assertion and on the form of the attack made by the opponent, either defend his previous claim or counter attack. The debate evolves following this pattern. The proponent wins the debate if he has the last word, i.e., the defence against one of the attack made by the opponent is a proposition that the opponent can not attack without violating the debate rules.

Dialogical logic was initially conceived by Lorenzen as a foundation for  intuitionistic logic (IL). Lorenzen idea was, roughly speaking, the following. It is possible to define a ``natural'' class of dialogue game in which given a formula $F$ of $IL$, i the proponent can always win a game on $F$, \emph{no matter how the opponent choose to attack his assertion in the debate}, if $F$ is IL-valid.  This intuition was formalized as the completeness of the dialogical proof system with respect to  provability or validity in any model:  

\begin{description}\item[\it Completeness of a dialogical proof system:] 
Given a logical language $L$ and a notion of validity for $L$ (either proof theoretical or model theoretical) and a notion of dialogical game, a formula $F$ is  valid in $L$  if and only if, there is a winning strategy for the proponent of $F$ in the class of games under consideration.  
\end{description} 

Where a winning strategy can be intuitively understood as an algorithm that take as input the moves of the  game made so far  and outputs proponents moves. Unfortunately, almost 40 years of work were needed to get a first correct proof of the completeness theorem \cite{Felscher2002}. Subsequently different systems of dialogic logic were developed. We present here a system of dialogical logic that is complete for classical first order logic.\\

\subsection{Language, Positive and Negative Subformulas}
We consider a first order language $\mathcal{L}$ in which the definition of terms and formulas is given as usual, and in which the negation of a formula $\neg F$ is defined by $F\Rightarrow \bot$ where $\bot$ is an arbitrary atomic formula.  Since we will use the notion of Gentzen's subformula and the positive (negative) occurrence of a subformula in a formula later, let us briefly detail these notions. 

\begin{enumerate}
    \item $F$ is a Gentzen-subformula of $F$
    \item if $F_1\star F_2$ is a Gentzen-subformula of $F$ so are $F_1,F_2$ for $\star\in \{\wedge,\vee,\Rightarrow\}$
    \item if $\neg F'$ is a Gentzen-subformula of $F$ so is $F'$
    \item if $Qx F'$ is a Gentzen-subformula of $F$ so is $F'[t/x]$ for all $t$ free for $x$ in $F'$ and $Q\in\{\forall,\exists \}$
\end{enumerate}

In the rest of the paper the term \emph{subformula} will always be used as a short-cut for the term Gentzen-subformula.\\
The notions of positive and negative subformulas are defined as follows 

\begin{enumerate}
    \item $F$ is a positive subformula of $F$
    \item if $F_1\wedge F_2$ or $F_1\vee F_2$ is a positive (resp negative) subformula of $F$ so are $F_1,F_2$
    \item if $Q x F'$ is a positive (resp. negative) subformula of $F$ so is $F'[t/x]$
    \item if $F_1\Rightarrow  F_2$  positive (resp negative) subformula of $F$, then $F_1$ is a negative
(resp. positive) subformula of $F$, and $F_2$ is a positive (resp. negative) subformula of
F
\item if $\neg F'$ is a positive (resp negative) subformula of $F$ then $F'$ is a negative (resp. positive) subformula of $F$

\end{enumerate}

\subsection{Attack Rules} Here we define the rules that permit to attack a formula and what counts as a defence against the attack. The symbols in the central column that are not formulas are called auxiliary symbols

 \begin{center}

 \begin{tabular}{l|l|l}
  Assertion   & Attack  & Defence  \\
   \hline 
   $F_1\Rightarrow F_2$ & $F_1$ & $F_2$\\
    $F_1\wedge F_2$  &$?\wedge_i$  & $F_i$\\
    $F_1\vee F_2$ & $?\vee $ & $F_i$\\
    $\forall x F$ & $?\forall [t/x]$ & $F[t/x]$\\
    $\exists x F$ & $?\exists $ & $F[t/x]$\\

\end{tabular}
\end{center}
 In the table $i\in\{1,2\}$, 
  $x$ stands for an arbitrary  variable, and $t$ stands for an arbitrary term. In the attack $?\forall [t/x]$   and in the defence $F[t/x]$,  $t$ is called the  chosen term. 
   For the sake of clarity we provide a paraphrase of the rules.

  \begin{description}
  \item[$\Rightarrow$] If one of the two participants in the dialogue asserts a conditional $F_1\Rightarrow F_2$, his opponent  may attack this assertion by asserting, in turn, the antecedent of the  conditional $F_1$ and ask him to assert $F_2$.
  The one who asserted the conditional can defend himself against this attack by asserting $F_2$, or, he can counterattack on $F_1$. If he chose the latter option, the roles of the two players would be reversed.

  \item[$\wedge$]  If one of the two participants in the dialogue affirms a conjunction, his opponent may attack the assertion by choosing one of the two members of the conjunction and asking to continue the dialogue by asserting the chosen member. 
The one who asserted the conjunction can defend himself by asserting the member chosen by his opponent. 
  \item[$\vee$] If one of the two participants in the dialogue asserts a disjunction, his opponent attacks the assertion by asking him to choose one of the two members and assert it. 
The one who  asserted the disjunction can defend himself by choosing one of the two members and asserting it. 
  
  \item[$\forall $] If one of the two participants in the dialogue asserts  universally quantified formula $\forall x F$ his opponent  may attack the formula by choosing a term $t$ of the language  and asking him to assert $F[t/x]$, i.e., the  formula in which the variable $x$ is replaced by the term $t$. The one  who has asserted the universally quantified formula can defend himself against such an attack by asserting $F[t/x]$. 
  
  \item[$\exists$]
  
  if one of the two participants in the dialogue asserts a existentially quantified formula $\exists x F$ his opponent  may attack the formula by asking him to choose a term  $t$ of the language  and assert $F[t/x]$, i.e.\ the  formula in which the variable $x$ is replaced by the term $t$. The one  who has asserted the existentially quantified formula can defend himself against such an attack by choosing a term $t$ and asserting $F[t/x]$.

  \end{description}
  
  Remark that the rule $\wedge $ and $\forall $ are similar: the one who attacks the assertion makes the choice of the subformula or of the term respectively. In the same way the rule for $\vee$ and $\exists$ are similar. The choice of the subformula, or of the term, is a responsibility of the one who defends.

\subsection{Rules of Dialogical Games}

\textbf{Moves}  are pairs $(i,s)$ with $i\in\{?,!\}$ and $s$ being either a formula or an auxiliary symbol. Moves  are called attacks whenever $i=?$ and defences whenever $i=!$. 

An \textbf{attack} $(\mathbf{?},F)$ 
with $F$ a formula is said to be an assertion.

Defences are assertions. 

%To memorize the assertions to which the attacks refer, and to memorize the attacks
%to which the defences refer, we consider attacks and responses indexed by integers
 A \emph{Pre-justified} sequence is given by a sequence of moves 
$\Vec{M}=M_0,M_1\ldots M_j\ldots$, together  with a partial  function $f:\Vec{M}\to \Vec{M}$  such that for all all $M_i$ in the sequence for which  $f$ is defined, $f(M_i)<M_i$. Let
\medskip

\centerline{$\Vec{M}=M_0,M_1\ldots M_j\ldots$}

An attack move $M_n$ of the form $(?,s)$ is said to be \textbf{justified} whenever  $f(M_n)$ is an assertion of  $F$  and $s$ is an attack of $F$. We say  that $f(M_n)$ is the \textbf{enabler} of $M_n$.

A defence $M_n$ of the form $(!,F)$ is \textbf{justified} if $f(M_n)=(?,s)$  is a justified attack of $M_j$ $M_j$ asserts $F'$, $s$ attacks $F'$ and $F$ is a defence from $s$. $M_i$ is the \textbf{enabler} of $M_n$. An assertion $M_n$ is a \textbf{reprise} if and only if there exists another move $M_j$   with smaller index and opposite parity that asserts the same formula. An assertion $M_n$ is called an \textbf{existential repetition}, if its enabler is of the form $(?,\exists)$ and there is another move of the same parity $M_n'$ with $n'<n$ having the same enabler. A pre-justified sequence in which each move is justified is called \emph{justified sequence}

\subsection{Games}

\begin{definition} 
 A \textbf{game}   for a formula $F$ is a pre-justified sequence 
 
 \[M_o,M_1,\ldots M_j\ldots \]
 such that 
 
\begin{enumerate}
    \item $M_o$ is ($!,F)$  and $M_1,\ldots M_j\ldots $ is  a justified sequence of in which each odd-index move is enabled by the its immediate predecessor and in which each even-index move is enabled by a preceding odd-index move
%    \item\label{well-bracketing}  if $M_j$ is a defence then the enabler of $M_j$, $M_k$, is the last attack that has not been defended the prefix of the suite ending with $M_{j-1}$
   
    \item if an even-index move asserts an atomic formula  then it is a reprise.
    
    \item for all even $m,n$  if  $M_m$ and $M_n$ are defence moves that asserts the same subformula occurrence of $F$ and are enabled by the same move $M_j$ then $m=n$

\end{enumerate}
\end{definition}

Odd-index moves are opponent moves ($\mathbf{O}$-moves) while even index moves are called proponent moves ($\mathbf{P}$-moves).
A move $M$ is \textbf{legal} for a game $G$ if the pointing sequence $G,M$ is a game.

A game $G$ is $\mathbf{won}$ by the proponent if, and only if,  it is finite and there is no $\mathbf{O}$-move which is legal for $D$. It is won by the opponent otherwise. 

Here are some properties needed in the proof of the completeness of our dialogical games. 

\begin{proposition}
Let $\Vec{M}=M_o,\ldots M_n$ be a game. $\Vec{M}$ is won by $\mathbf{P}$  only if, $M_n$ is the assertion of some atomic formula $a(t_1,\ldots t_m)$
\end{proposition}

\begin{proposition}
For all game $G$, for all formula $F$ for all subformula $F'$ of $F$. If there is an $\mathbf{P}$-move ($\mathbf{O}$-move) in $G$ that asserts $F'$ then $F'$ is  a  positive (negative) subformula of $F$. 
\end{proposition}

\begin{proof}
Let $F$ be any formula and $G$ any game.  We show the proposition by induction on the length of $G$. If the length is $1$ then $G$ consist of only one move that is a $\mathbf{P}$-move asserting a formula $F$ and so the proposition holds. \\
Suppose that the proposition holds for all games $G$ having length $n$ and  let $G'$ be a game having length $n+1$. Let $M_n$ be the last move of $G'$. Suppose that $M_n$ is a $\mathbf{P}$-move (the argument for $\mathbf{O}$-moves runs in a very similar  way)  We have three cases. 

\begin{enumerate}
    \item If $M_n$ is not an assertion the proposition holds automatically by induction hypothesis. 
    \item if $M_n$ is a defence asserting   some formula  $F'$. Since $G'$ is a game the sequence $M_1\ldots M_n$ is justified. Thus $M_n$ is enabled by some $\mathbf{O}$-move $M_k$ with $(k<n)$. If $M_k:=(?,F_1)$ (the other cases are easier) then it is an attack against $M_{k-1}$ and $M_{k_1}$ is a $\mathbf{P}$-move that asserts $F_1\Rightarrow F'$. By induction hypothesis $F_1\Rightarrow F'$ is a positive subformula of $F$, and $F_1$ is a negative subformula of $F$. Thus $F'$ is a positive subformula of $F$ by definition. 
    
\item If $M_n$ is an assertion and an attack, let $F'$ be the asserted formula. As before there must exists an enabler of $M_n$, call it $M_k$ $(k<n)$, $M_k$ is necessarily an $\mathbf{O}$-move  that asserts the formula $F'\Rightarrow F''$. By induction hypothesis this last formula is a negative subformula of $F$, thus $F'$ is a positive subformula of $F$ by definition. 
\end{enumerate}
\end{proof}

An easy consequence of the latter proposition is the following 

\begin{proposition}
Let $G$ be a game for a formula $F$ and let $M_n$ be a $\mathbf{P}$-move in $G$ asserting an atomic formula $a(t_1,\ldots t_m)$. Then this latter formula appears both as a negative and positive subformula of $F$
\end{proposition}

\subsection{Strategies}

Informally speaking, a strategy for a player  is an  algorithm for playing the game that tells the player what to do for every possible situation throughout the game. 
We informally describe how a strategy should operate and then formalize this notion. Imagine being engaged in a 
game $G$, that the last move of $G$ was played according to the strategy, and that it is now your opponent's turn to play. Your opponent could extend the game in different ways: for example if you are playing chess, you are  white and you just made your first move by moving a pawn to a certain position of the chessboard,  black  can in turn move a pawn or move a  horse. If you are playing according to the strategy, the strategy should tell you how to react against either type of move. If black moves a pawn to $C6$ and you just moved your pawn to $C3$ then move the horse to $H3$. If black moves a horse to $H6$ and you just moved your pawn to $C3$ then move  you pawn in $B4$. This type of pattern can be visualized as a tree in which each node is a move in the game, the moves of my opponent have at most one daughter, and my moves have as many daughters as there are available moves for my opponent. 
In fact it is rather standard, in the game semantics/dialogical logic literature to define a strategy as a tree of games \cite{Felscher2002,mcCusker,hylandGame}. We thus formalize the notion of strategy as follows. 
Given a game $G$ we say that a variable $x$ appears in the game iff $x$ appears in some asserted formula or is free in  the choice of some universal attack. Let $(v_i)_{i\in I}$ be an enumeration of the variables in $L$. 
A $\mathbf{strategy}$ $S$ is a prefix-closed set of games for the same formula (i.e. a tree of games for the same formula)  such that 
    \begin{enumerate}
    \item If $G$ belongs to the strategy and the last move of $G$ is a $\mathbf{P}$ move that neither  an assertion of a universally quantified formula nor an existential attack then for all move $M$ legal for $G$, $G,M$ belongs to the strategy
    
 \item if $G,M$ and $G,M'$ belong to $S$ and $M,M'$ are $\mathbf{P}$-moves then $M=M'$
\item \label{universal} if $G,M$ and $G,M'$ belongs to $S$ and $M,M'$ are $\mathbf{O}$-moves and universal attacks then $M=M'$ moreover the chosen variable is the first variable in the enumeration that does not appear in $G$
\item if $D,M$ and $D,M'$ belongs to $S$ and $M,M'$ are $\mathbf{O}$-moves and existential defences then $M=M'$ moreover the term chosen to defend is the first variable in the enumeration that does not appear in $G$
\item\label{no-delay }if $G$ belongs to $S$ and the last move of $G$ is $\mathbf{O}$-move that is an attack against an existential quantifier then $G,M$ belongs to $S$, where $M$ is a defence against the last move of $G$
 \end{enumerate}

    A strategy $S$ is $\mathbf{P}$-winning iff each game in $S$ is won by $\mathbf{P}$. Considering that a strategy is a tree of games,  in what follows we will sometimes speak of nodes of a strategy as a shortcut for moves of a game $G$ that belongs to the strategy. 
    
   \subsection{Validity} 
   
   \begin{definition} 
   Given a first order formula $F$ we say that $F$ is \emph{dialogically valid} if, and only if, there exists a winning strategy $S$ for the the formula. 
   \end{definition}

    Below are three examples of winning strategies. 
    The  blue dotted arrow represent the function $f$ that points back from $\mathbf{P}$ moves to the $\mathbf{O}$ move that enables them. keep in mind that every  $\mathbf{O}$-move   is enabled by the immediately preceding move. 

\begin{figure} 
%\caption{Winning strategy for $\forall x a(x) \vee \exists x \neg a(x)$}  
    \begin{tikzpicture}[scale=1.3,->,level/.style={sibling distance = 8.5cm,
  level distance = 0.6cm}] 
\node  {\footnotesize{$!,\forall x a(x) \vee \exists x \neg a(x) 
            $}}
    child{ node  (1){\footnotesize {$?\vee $}}
      child{ node (2) {\footnotesize{$!,\forall x a(x) $}} 
            	child{ node (3)  {\footnotesize{$?,\forall[w/x] $}}
        child{node (4){\footnotesize{$!,\exists x \neg a(x) $}}
        child{node  (5){\footnotesize {$?,\exists$}}
              child{node(6)
        {\footnotesize{$!,\neg a(w)$}}  
        child{node  (7){\footnotesize {$?,a(w)$}}
        child{node (8)   {\footnotesize{ $!,a(w)$ }}
         }}}}}}}};
      \draw[thick,blue,dotted] (2) to[out=180,in=180](1); 
        \draw[thick,blue,dotted] (4) to[out=0,in=0](1);
         \draw[thick,blue,dotted] (6) to[out=180,in=180](5);
          \draw[thick,blue,dotted] (8) to[out=180,in=180](3); 
          
; 
\end{tikzpicture}
\hfill 
%\label{strategy1} 
%\end{figure} 
%\begin{figure} 
%\caption{Winning strategy for  $\exists x (a(x)\Rightarrow \forall y a(y))$} 
\begin{tikzpicture}[scale=1.3,->,level/.style={sibling distance = 8.5cm,
  level distance = 0.6cm}] 
\node  {\footnotesize{$!,\exists x (a(x)\Rightarrow \forall y a(y))
            $}}
    child{ node  (1){\footnotesize {$?,\exists $}}
      child{ node (2) {\footnotesize{$!,\ a(c)\Rightarrow    \forall y a(y) $}} 
            	child{ node (3)  {\footnotesize{$?,a(c) $}}
        child{node (4){\footnotesize{$!,\forall y a(y) $}}
        child{node  (5){\footnotesize {$?,\forall[w/y]$}}
              child{node(6)
        {\footnotesize{$!,a(w)\Rightarrow\forall y a(y)$}}  
        child{node  (7){\footnotesize {$?,a(w)$}}
        child{node (8)   {\footnotesize{ $!,a(w)$ }}
         }}}}}}}};
      \draw[thick,blue,dotted] (2) to[out=180,in=180](1); 
        \draw[thick,blue,dotted] (4) to[out=180,in=180](3);
         \draw[thick,blue,dotted] (6) to[out=180,in=180](1);
          \draw[thick,blue,dotted] (8) to[out=0,in=0](5); 
          
; 
\end{tikzpicture}

%\label{strategy2} 
%\end{figure} 
%
%\begin{figure} 
%\caption{Winning strategy for $\forall x (a(x)\wedge b(x))\Rightarrow (\forall x a(x))\wedge (\forall x b(x))  

\vspace{1cm}

\begin{center} 
\begin{tikzpicture}[scale=1.3,->,level/.style={sibling distance = 2cm,
  level distance = 0.6cm}] 
\node  {\footnotesize{$!,\forall x (a(x)\wedge b(x))\Rightarrow (\forall x a(x))\wedge (\forall x b(x))  
            $}}
    child{ node  (1){\footnotesize {$[\forall x (a(x)\wedge b(x)) $}}
      child{ node (2) {\footnotesize{$!(\forall x a(x))\wedge (\forall x b(x)) $}} 
           child{node(3){\footnotesize{$[?\wedge_1$}}
    child{node (4) {\footnotesize{$!\forall x a(x)$ }}
    child{node(5){\footnotesize{$?,\forall [w/x]$}}
    child{node (6) {\footnotesize{$\ ?\forall[w/x]$}}
    child{node (7) {\footnotesize{$! ,a(w)\wedge b(w))$}}
    child{node (8) {\footnotesize{$?,\wedge_1$}}
    child{node (9) {\footnotesize{$!,a(w)$}}
     child{node(10){\footnotesize{$!a(w)$}}}}}}}}}}
    child{node  (11){\footnotesize{$?,\wedge_2$}}
     child{node (12) {\footnotesize{$!,\forall x b(x)$ }}
    child{node(13){\footnotesize{$?,\forall [w/x]$}}
    child{node (14) {\footnotesize{$?,\forall[w/x]$}}
    child{node (15) {\footnotesize{$!,a(w)\wedge b(w)$}} 
    child{node (16) {\footnotesize{$?,\wedge_2$}}
    child{node (17) {\footnotesize{$!,b(w)$}} 
    child{node (18) {\footnotesize{$!,b(w)$}}}}}}}}}}}};
            
      \draw[thick,blue,dotted] (2) to[out=-180,in=180](1); 
      \draw[thick,blue,dotted] (4) to[out=-180,in=180](3);
      \draw[thick,blue,dotted] (6) to[out=-180,in=180](1);
      \draw[thick,blue,dotted] (8) to[out=-180,in=180](7); 
         \draw[thick,blue,dotted] (10) to[out=180,in=180](5); 
       \draw[thick,blue,dotted] (12) to[out=0,in=0](11); 
        \draw[thick,blue,dotted] (14) to[out=0,in=0](1);
         \draw[thick,blue] (16) to[out=0,in=0](15);
          \draw[thick,blue,dotted] (18) to[out=0,in=0](13);
         \end{tikzpicture} 
%\label{strategy3} 
\end{center} 
\caption{Three winning strategies. 
} 
\end{figure}
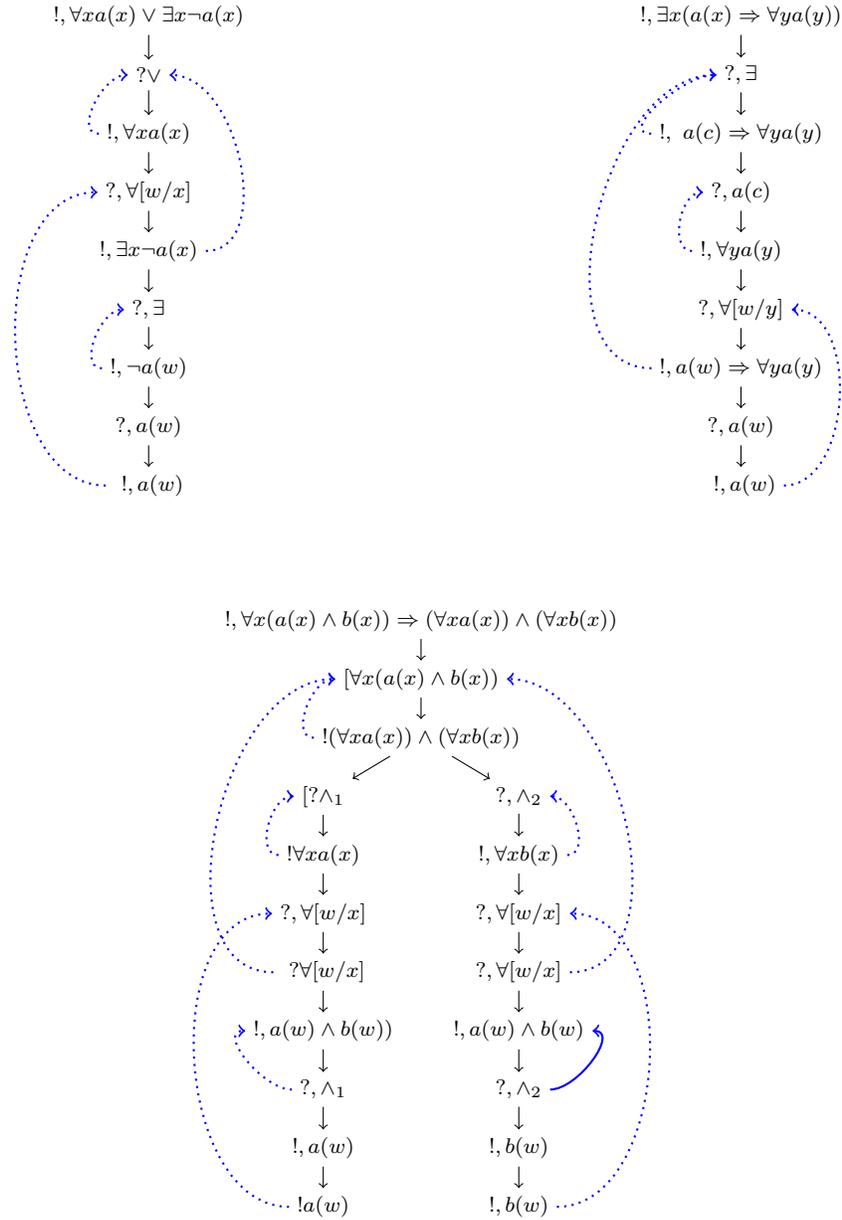

\section{Dialogical Validity is Classical Validity}

In this section we show the equivalence, for a formula $F$ between the dialogical validity of $F$ (the existence of winning strategy for the proponent of $F$)  and classical validity of $F$ ($F$ being true in all interpretations). To prove this equivalence, we use a particular version of the sequent calculus, GKK (see the table below). The calculus GKK is a variant of the sequent calculus  GKm \cite{TvD96basicPT},  an extension  to first order of the propositional calculus LKQ by Herbelin  \cite{Her95A}: all those calculi are complete for classical logic (first order for GKK and GKm, propositional for LKQ), in the sense that 
a sequent $\Gamma\vdash \Delta$ can be derived if and only if the formula $\bigwedge \Gamma \Rightarrow \bigvee \Delta$ classically valid. 

Later on, as Herbelin does for LKQ, we shall consider a  restriction (\emph{strategic} derivations) on the use of the left implication introduction rule, a restriction which does not affect the completeness of the sequent calculus with respect to validity. 

\newpage 

\begin{definition}
The sequent calculus GKK is defined by the  rules below: 
%\begin{figure} 
%\label{gkk} 
%\caption{The GKK sequent calculus} 

\begin{center} 
\footnotesize 
\begin{tabular}{p{0.40\textwidth}p{0.48\textwidth}}
\multicolumn{2}{c}{}{\begin{minipage}{0.6\textwidth} 
\begin{prooftree}
\AxiomC{}\RightLabel{$Id$}
\UnaryInfC{$\Gamma, A\vdash \mathbf{A},\Delta$}
\end{prooftree}\end{minipage}} 
\cr  
\begin{prooftree}
        \AxiomC{$\Gamma, A\vdash B,\Delta$}\RightLabel{$\Rightarrow R$}
        \UnaryInfC{$\Gamma\vdash \mathbf{A\Rightarrow B},\Delta$} 
        \end{prooftree}
        &
\begin{prooftree}
\AxiomC{$\Gamma,A\Rightarrow B\vdash A,\Delta$}
\AxiomC{$\Gamma,A\Rightarrow B, B\vdash \Delta$}\RightLabel{$\Rightarrow L$}
\BinaryInfC{$\Gamma, \mathbf{A\Rightarrow B}\vdash \Delta$}
\end{prooftree}
\cr  
\begin{prooftree}
\AxiomC{$\Gamma\vdash A,\Delta$}
\AxiomC{$\Gamma\vdash B,\Delta $}\RightLabel{$\wedge R$}
\BinaryInfC{$\Gamma\vdash \mathbf{A\wedge B},\Delta$}
\end{prooftree}
& 
\begin{tabular}[b]{p{0.28\textwidth}p{0.24\textwidth}}
\begin{prooftree}
\AxiomC{$\Gamma, A,A\wedge B\vdash \Delta$}\RightLabel{$\wedge L_1$}
\UnaryInfC{$\Gamma,\mathbf{A\wedge B}\vdash \Delta$}
\end{prooftree}
& 
\begin{prooftree}
\AxiomC{$\Gamma, B,A\wedge B\vdash \Delta$}\RightLabel{$\wedge L_2$}
\UnaryInfC{$\Gamma,\mathbf{A\wedge B}\vdash \Delta$}
\end{prooftree}
\end{tabular} 
\cr 
\begin{prooftree}
\AxiomC{$\Gamma\vdash A,B,\Delta$}\RightLabel{$\vee R$}
\UnaryInfC{$\Gamma\vdash \mathbf{A\vee B,}\Delta$}
\end{prooftree}
&
\begin{prooftree}
\AxiomC{$\Gamma, A\vee B, A\vdash \Delta$}
\AxiomC{$\Gamma, A\vee B, B\vdash \Delta $}\RightLabel{$\vee L$}
\BinaryInfC{$\Gamma,\mathbf{A\vee B}\vdash \Delta$}
\end{prooftree}
\cr 
\begin{prooftree}
\AxiomC{$\Gamma \vdash \exists x A, A[t/x],\Delta$}
\UnaryInfC{$\Gamma \vdash \mathbf{ \exists x A},\Delta $}
\end{prooftree}
& 
\begin{prooftree}
\AxiomC{$\Gamma,A[y/x],\exists x A\vdash \Delta$}
\UnaryInfC{$\Gamma,\mathbf{\exists x A}\vdash \Delta$}
\end{prooftree} 
\cr 
\begin{prooftree}
\AxiomC{$\Gamma\vdash A[y/x],\Delta$}\RightLabel{$\forall R$}
\UnaryInfC{$\Gamma\vdash \mathbf{ \forall x A},\Delta$}
\end{prooftree}
&
\begin{prooftree}
\AxiomC{$\Gamma A(t), \forall x  A \vdash \Delta$}\RightLabel{$\forall L$}
\UnaryInfC{$\Gamma,\mathbf{\forall x A}\vdash \Delta$}
\end{prooftree}
\end{tabular} 
 \end{center} 
\end{definition}

 Here are some comments and terminology on the calculus GKK: 
 \begin{itemize} 
 \item Greek upper-case letters $\Gamma,\Delta,\ldots$ stand for multisets of formulas. 
   \item 
 In the $Id$-rules  $A$ is an \underline{atomic} formula. Since this rule has no premisse it is also called an \emph{axiom}. 
 \item 
 In the rules $\forall R$ and $\exists L$ the variable $y$ does not occur anymore in their conclusion sequents.  
 \item 
The bold formulas in the conclusion of each rules are called \emph{active formulas}. 
\item 
A derivation $\pi$ of a sequent $\Gamma\vdash \Delta$ is a tree of sequents constructed according to the above rules, in which leaves are axioms  $Id$-rules. 
\item 
We say that a sequent $\Gamma\vdash \Delta$ is derivable or provable whenever there exists a derivation of it from axioms.  
\item 
A derivation $\pi$ of $\Gamma\vdash \Delta$ with $\Gamma$ empty and $\Delta$ reduces to one formula $F$, is said to be a \emph{proof} of $F$. 
\end{itemize} 

\begin{proposition}
A  sequent  $\Gamma\vdash \Delta$ is derivable in the  sequent calculus  GKK if and only if  is derivable in Gkm. 
\end{proposition}

\begin{proof}[Sketch]
One just shows that each rule in the above sequent calculus is admissible in Gkm and vice-versa, using contraction and weakening admissibility, adapting the proof from e.g.\ \cite{TvD96basicPT}. 
\end{proof}

\subsection{From Strategies to Derivations}

Because sequent calculus is complete for classical logic, the following proposition shows that a formula $F$ with a winning strategy is true in any interpretation. 

\begin{proposition}\label{strat2proof} 
Given a formula $F$, if there is a winning strategy for $F$ then there is a proof $\pi$ of $F$ in sequent calculus GKK. 
\end{proposition}

% LE CALUL DES SEQUENT N ETAIT PAS PRECISE 

To do so, we associate sequents  to nodes  of a strategy and we prove that whenever the sequents associated with  nodes  of height $n$ are provable, so are those of height $n+1$. So let's start by associating sequents to \emph{some} nodes of a strategy. To do this we focus on nodes of a strategy that are $\mathbf{O}$ moves.\\
 
 % DANS LA STRATEGIE: TU MA DIT QU4IL N4Y A QUE DES O MOVES? CE N EST PAS COHERENT !!!!!
 
 Let $S$ be a strategy and  $G$ be a game in $S$.   Consider the sequence of moves $\Vec{M }$ in $G$ obtained by  forgetting all its   $\mathbf{P}$-moves i.e  if $G$ is $M_0,M_1,\cdots M_{2n}\cdots $  then  $\Vec{M}$ is 
    \[M_1,M_3,\ldots M_{2n-1}\ldots \]
    
    Call $D(S)$ the prefix closed set (i.e. the tree) of such sequences containing the empty sequence $\epsilon$. 
   Given a strategy $S$ we associate a sequent  $\Gamma_{\vec{M}}\vdash \Delta_{\vec{M}}$  to each $\Vec{M}\in D(S)$ in the following way 
   
   \begin{enumerate}

       \item To the empty sequence $\epsilon$ we associate $F$. 
          \item if the sequence ends in an assertion move 
          \begin{enumerate}
              
              \item if the assertion is a defence move $(!,F)$ then we associate the sequent $\Gamma,F\vdash \Delta $ where $\Gamma\vdash \Delta$ is the sequent associated to the prefix of the sequence 
              \item  if the assertion is an attack move $(?,F)$ against (in the game ) an assertion of $F\rightarrow C$ then we associate the sequent $\Gamma,F\vdash C,\Delta$ where $\Gamma\vdash \Delta$ is the sequent associated to the prefix of the sequence from which we have erased the formula $F\Rightarrow C$ on the right of $\vdash$
          \end{enumerate}
       \item if the sequence end in a move that is not an assertion then it should be an attack, $(?,s)$ where $s$ is an auxiliary symbol. We have three type of case 
       
       \begin{enumerate}
           \item If $s$ is either $?\vee,\wedge_1$ or $\wedge_2$ then  we associate the sequent $\Gamma\vdash \Delta'$ where $\Gamma$ is equal to the $\Gamma$ associated to the prefix of the sequence, and $\Delta'$ is the sequent $\Delta$ associated to the prefix of the sequence from which we have erased the formula $C$ asserted by the move that is attacked by $(?,s)$ and to which we have added the sub-formula  of $C$ chosen by $\mathbf{O}$ in the case of $\wedge_i$ and both subformulas of $C$ is $s$ is $\vee$
            \item if $s$ is $\forall[w/x]$ then we associate the sequent $\Gamma\vdash A[w/x],\Delta'$ where $A[w/x]$ is the subformula of the formula $\forall x A$ asserted by $\mathbf{P}$ and attacked by $(?,s)$ and $\Delta $ is obtained by erasing $\forall x A$ from  $\Delta$ in the sequent $\Gamma\vdash \Delta$  associated to the prefix of the sequence. 
            \item if $s$ is $\exists$ then we associate the sequent $\Gamma\vdash F(t),\Delta$ where $F(t)$ is the formula asserted by $\mathbf{P}$ in is defence against $(?,s)$ and $\Gamma\vdash \Delta$ is the sequent associated to the prefix of the sequence. Remark that the $\mathbf{P}$-defence must exists by the definition of strategy.  
       \end{enumerate}
   \end{enumerate}

Given a winning strategy $S$, observe that  $D(S)$  can be viewed as a (non empty) finite tree $(D(S),\prec)$ with $\prec$ defined by $\vec{(M)}\prec \vec{M'}$ if and only if $\vec{M}$ is a suffix of $\vec{M}$ --- $\prec$ is the order associated with the tree structure on $D(S)$.    

The above procedure yields  $D(S)$,  a tree node-labeled with sequents --- some object similar to a sequent calculus proof -- such that: 
\begin{itemize} 
\item 
the  leaf-labels are of the form  $\Gamma,A\vdash A,\Delta$ with $A$ an atomic proposition, 
\item 
any formula 
% WHICHI ONE ?  ALL? 
in a sequent labeling a node of depth $n$  is a subformula of some formula  of the sequent labeling the nodes of depth $n+1$. 
\end{itemize} 

The following lemma assures us that the labelled tree respects the variable restriction on $\forall R$ $\exists L$

\begin{lemma}
Let $S$ be a winning strategy and let  $\vec{M}$ be a sequence in $D(S)$.
% A SEQUENCE OF SEQUENTS?????? 
Suppose that $\vec{M}$ ends in a move that is an attack against a universal quantifier $(?,\forall[w/x])$ or a defence against an existential attack $(!,A[w/x])$. Then the variable $w$ does not appears free in the sequent associated to the proper prefix of $\vec{M}$
\end{lemma}

A decorated tree $D(S)$ of a strategy $S$  really looks like a derivation in the sequent calculus GKK. 
% LE CALCUL DES SEQUENTS N ETAIT PAS PRECISE

We are now ready to prove the following lemma, which clearly entails  proposition \ref{strat2proof} above:  

\begin{lemma}
To each $\Vec{M}\in D(S)$ we can associate a derivation $\pi_{\Vec{M}}$ of $\Gamma_{\Vec{M}}\vdash  \Delta_{\Vec{M}}$ --- the sequent associated to $\vec{M}$. 
\end{lemma}

\begin{proof}
By well-founded induction on $(D(S),\prec)$. Suppose that for each suffix $\Vec{M}$ of $\vec{M'}$ the proposition holds. We associate a derivation to $\Vec{M'}$ by considering the last move of $M_{2_n}$ $\Phi(\Vec{M})$. Where $\Phi{\vec{M}}$ is the unique game $G$ in $S$ ending in a $\mathbf{P}$-move such that $\vec{M}$ is obtained from $G$ by  erasing $\mathbf{P}$-moves. 

We only present here some non straightforward cases: 
\begin{enumerate}
    \item if $M_{2_n}$ is an attack $(?,A)$ on the assertion $A\Rightarrow C$ depending on the form of $A$
    \begin{itemize}
       
        \item if $A$ is atomic then the immediate suffix of $\vec{M}$ is $\Vec{M}(!,C)$ for which the proposition hold by hypothesis. We associate it with the following derivation 
        
   \begin{footnotesize}
        
        \begin{prooftree}
        \AxiomC{}
        \UnaryInfC{$\Gamma, A\Rightarrow C,A\vdash A$}
        \AxiomC{$\vdots \pi_{\Vec{M}(!,C)}$}
        \noLine
        \UnaryInfC{$\Gamma,A\Rightarrow  C,C\vdash F$}
        \BinaryInfC{$\Gamma,A\Rightarrow C\vdash F$}
        \end{prooftree}
      \end{footnotesize}
       \item if $A=(A_1{\Rightarrow}A_2)$  then $\Vec{M}$ has  two immediate  suffixes namely $\vec{M},(?,A_1)$ and $\Vec{M},(!,C)$, for which the proposition holds by hypothesis.  We associate the following derivation: 
       
        \begin{footnotesize}
        \begin{prooftree}
        \AxiomC{$\vdots \pi_{\Vec{M'}(?,A_1)}$}
        \noLine
        \UnaryInfC{$\Gamma,(A_1\Rightarrow A_2)\Rightarrow C,A_1\vdash A_2$}
        \UnaryInfC{$\Gamma,(A_1\Rightarrow A_2)\Rightarrow C\vdash A_1\Rightarrow A_2$}
        \AxiomC{$\vdots\pi_{\Vec{M'}(!,C)}$ }
       \noLine \UnaryInfC{$\Gamma,(A_1 \Rightarrow A_2)\Rightarrow C,C\vdash F$}
        \BinaryInfC{$\Gamma,(A_1 \Rightarrow A_2)\Rightarrow C\vdash F$}
        
        \end{prooftree}
        \medskip
        \medskip
         \end{footnotesize}
         
         \item  If $M_{2n}$ is an existential repetition asserting a formula $F[t/x]$ we proceed as follows: we only consider  the case where $F[t/x]:=(F_1\Rightarrow F_2)[t/x]$,  By induction hypothesis there is  derivation of the sequent $\Gamma F_1\vdash \exists x F_1[t/x]\Rightarrow F_2[t/x],\Delta$ associated to the direct suffix $\Vec{M}(?,F_1)$ of $\Vec{M}$. We thus associate the following derivation 

\begin{footnotesize}
 \begin{prooftree}
 \AxiomC{$\vdots\pi_{\vec{M}(?,F_1)}$}
 \noLine
 \UnaryInfC{$\Gamma F_1[t/x]\vdash \exists x (F_1 \Rightarrow F_2),F_2[t/x] ,\Delta$}\RightLabel{$\Rightarrow R$}
 \UnaryInfC{$\Gamma\vdash \exists x (F_1 \Rightarrow F_2), (F_1\Rightarrow F_2)[t/x],\Delta$ }\RightLabel{$\exists R$}
 \UnaryInfC{$\Gamma\vdash \exists x( F_1\Rightarrow F_2),\Delta $}
 \end{prooftree}
\end{footnotesize}

        \end{itemize}
  \end{enumerate}      
    
\end{proof}

This proves   proposition \ref{strat2proof} and thus assures us that that if a formula is dialogically valid then it is provable in sequent calculus, hence true in any interpretation.

\subsection{From Derivations to Strategies}

We have just shown that if a formula is dialogically valid  then it is provable in sequent calculus GKK. We now show the converse, by turning a GKK sequent calculus derivation into a winning a strategy, but we shall impose a restriction on GKK derivations --- a restriction which derives exactly the same sequents.   

Indeed not all derivations in GKK are the image of some winning strategy. For instance,  the two derivations below where $c(x),a,b$ are atomic formulas  are not the image of any winning  strategy $S$ although there are winning strategies for the two formulas (bold formulas are active occurences of formulas in the sequent):  

\medskip
\medskip

\begin{footnotesize}

\begin{minipage}{0.40\textwidth}

\begin{prooftree}

\AxiomC{$c(x)\vdash \mathbf{c(x)}$}
\UnaryInfC{$\mathbf{\forall x c} \vdash c(x)$}
\UnaryInfC{$\forall x c \vdash \mathbf{\exists x c}$}
\end{prooftree}
\end{minipage}

\begin{minipage}{0.49\textwidth}
\begin{prooftree}

\AxiomC{$a\Rightarrow b, a\vdash \mathbf{a},b$}
\UnaryInfC{$a\Rightarrow b\vdash a, \mathbf{a\rightarrow b}$}
\AxiomC{$b,a\Rightarrow b,a\vdash \mathbf{ b}$}
\UnaryInfC{$b,a\Rightarrow b\vdash \mathbf{a\Rightarrow b}$}
\BinaryInfC{$\mathbf{a\Rightarrow b} \vdash a\Rightarrow b$}
\UnaryInfC{$\vdash (a\Rightarrow b)\Rightarrow (a\Rightarrow b)$}

\end{prooftree}
    \end{minipage}

\end{footnotesize}

 \medskip

This leads us to restrict proofs of GKK to strategic proofs  which derive the same sequents but always correspond to winning strategies, and to proceed as follows: 
 
  \begin{itemize}
      \item  We first give a very informal description of the procedure that we use to transform a derivation into a strategy 
      
      \item By looking on how the derivation should be made in order for the procedure to be successful we define a subclass of derivations of $GKK$ called $\emph{strategic derivations}$
      \item We show that the subclass is complete, in the sense that if the sequent $\Gamma\vdash \Delta$ is provable then it corresponds to a strategic derivation. 
  \end{itemize}

  % IL MANQUE DES PHRASE ICI
  Assume we already 
   converted a derivation into a tree of moves, from conclusion to premises,  up to a certain point by associating with each active formula, a set of sequences of moves.  Assume  that for each node $a$ of depth $n$ in the derivation, the branch of the derivation from $a$ to the root is associated with a game $G$ for the formula that is the root of the derivation. Let $a_1,\cdots, a_k$ be the the sons of $a$ having depth $n+1$. 
   
   We proceed as follows, starting with $a_1$
  \begin{enumerate}
      \item if $a_1$ bears a sequent in which the active formula $A$ is on the right then it is obtained by a right rule or an id rule. Add to $G$ one move by $\mathbf{P}$ that asserts $A$ obtaining $G'$. If these last assertion can be attacked by $O$ enumerate the attacks $M_i$ such that $G',M_i$ is a game and obtain $i$ sequences of moves that are games  
      \item if $a_1$'s label is a sequent in which the active formula $A$ is on the left then $a_1$ is obtained by a left rule. Add to 
      $G$ one move by $\mathbf{P}$ that consist in attacking $A$ obtaining a game $G'$ (if a is a conjunction or a universal quantifier look at the premise of $a'$ in order to find the right conjunct/term). Continue by extending $G'$ to (respectively) $G'M_1,\ldots G',M_n$ where the $M_i$ are either defences against the last move of $G'$ or counter attack and are all $\mathbf{O}$-moves
  \end{enumerate}
  
Such a procedure clearly ends. 

  If in $(1)$ the active formula is an existentially quantified formula $\exists x B $ then $\mathbf{P}$ asserts the formula and next it is attacked by $\mathbf{O}$ with $(?,\exists)$. By  definition of a strategy, $\mathbf{P}$ has to asserts $B[t/x]$. This means that $a_1$ should have just one son $a_2$ in which the formula $B[t/x]$ is active. 
  
  A similar  situation occurs in (2) when the active formula is a conditional $A\Rightarrow B$: $\mathbf{P}$ has to assert $A$, so $A$ must be the active formula of the left premise of the $\Rightarrow L$ rule.
  
  In order to overcome this problem, we introduce the following definition: 
  
 \begin{definition} (Strategic derivations) 
  A derivation $\pi$ in $GKK$ is said to be  strategic whenever the two conditions below hold:   
 \begin{itemize} 
 
 \item 
  
  for each application of a left rule, the formula occurrence $A$ in the left-hand premise is active.
  \begin{footnotesize}

  \begin{prooftree}
\AxiomC{$\Gamma,A\Rightarrow B\vdash A,\Delta$}
\AxiomC{$\Gamma,A\Rightarrow B, B\vdash \Delta$}\RightLabel{$\Rightarrow L$}
\BinaryInfC{$\Gamma, \mathbf{A\Rightarrow B}\vdash \Delta$}

\end{prooftree}

\end{footnotesize}

\item  for each $\exists R$ rule application the formula occurence $A[t/x]$ is active in the premise: 

\begin{footnotesize}

\begin{prooftree}
\AxiomC{$\Gamma \vdash \exists x A, A[t/x],\Delta$}
\UnaryInfC{$\Gamma \vdash \mathbf{ \exists x A},\Delta $}
\end{prooftree}

  \end{footnotesize}

\end{itemize} 
\end{definition}  
  
 \begin{proposition}
 If $\pi$ is a strategic derivation of $F$ then the procedure above outputs a winning strategy for $F$
 \end{proposition}
 
 Given these last proposition we can conclude our proof by the following 
 
 \begin{lemma}
For any multiset of formulas $\Gamma,\Delta$ there is a derivation of the sequent $\Gamma\vdash \Delta$ iff and only iff there is a strategic derivation of $\Gamma\vdash \Delta $
 \end{lemma}
 
 \begin{proof}
 The direction from right to left is straightforward:  each strategic derivation is a derivation in GKK. 
 
 The other direction results from a structural induction on the derivation $\pi$ in GKK. All cases are 
straightforward except when $\pi$ is obtained by $\pi'$ by the application of a $\Rightarrow L$ rule or an $\exists R$ rule. 
Let us discuss the $\exists R$ rule, which together with a similar result in  \cite{Her95A} dealing with all the propositional cases entail our proposition.

If $\pi $ ends in a $\exists R$ rule application then, by induction hypothesis there is a strategic derivation $\pi_1$ of its premise $\Gamma\vdash A[t/x], \exists x A, \Delta$. if $A[t/x]$ is active we are done. If not we can suppose, without loss of generality that the rule application in which $A[t/x]$ is active is just above the last rule $R$ of $\pi_1$. The "hard" case is when $R$ is a $\forall R$ rules, $A[t/x]= \exists y B (t,y) $ i.e. $\pi_1$ has the following shape:  

\begin{footnotesize}
\begin{prooftree}
\AxiomC{$\vdots$}
\noLine
\UnaryInfC{$\Gamma\vdash B(t,t'),A(y)$}
\UnaryInfC{$\Gamma\vdash \exists y B(t,y),D(y)$}
\UnaryInfC{$\Gamma\vdash \exists B(t,y),\forall w D$}

\end{prooftree}
\end{footnotesize}

The problem being that the term $t'$ can contain a free occurrence of $y$. In this case we let permute the $\exists R$ upwards in this way:

\begin{footnotesize}

\begin{prooftree}

\AxiomC{$\vdots$}
\noLine
\UnaryInfC{$\Gamma\vdash B(t,t'),A(y),\Delta'$}
\UnaryInfC{$\Gamma\vdash \exists y B(t,y),D(y),\Delta$}
\UnaryInfC{$\Gamma\vdash \exists B(t,y),D(y),\Delta'$}
\UnaryInfC{$\Gamma\vdash \exists x A, D(y),\Delta'$}
\UnaryInfC{$\Gamma\vdash \exists x A, \forall w D, \Delta' $}

\end{prooftree}

\end{footnotesize}
That way we obtain a strategic proof we wanted. 
 \end{proof}

  This concludes our proof of the equivalence between winning strategies for our dialogical games and the existence of a  proof in classical logic (here viewed, without lost of generality, as a strategic GKK proof).

\section{Categorial Grammars and Automated Theorem Proving}

Type-logical grammars are a family of frameworks for the analysis of natural language based on logic and type theory. Type-logical grammars are generally fragments of intuitionistic linear logic, with the Curry-Howard isomorphism of intuitionistic logic serving as the syntax-semantics interface. Figure~\ref{fig:arch} shows the standard architecture of type-logical grammars.

\begin{enumerate}
\item given some input text, a \emph{lexicon} translates words into formulas the result is a judgment in some logical calculus, such as the Lambek calculus or some variant/extension of it,
    \item the grammaticality of a sentence corresponds to the provability of this statement in the given logic (where different proofs can correspond to different interpretations/readings of a sentence),
    \item there is a forgetful mapping from the grammaticality proof into a proof of multiplicative, intuitionistic linear logic,
    \item by the Curry-Howard isomorphism, this produces a linear lambda term representing the derivational meaning of the sentence (that is, it provides instructions for how to compose the meanings of the individual words),
    \item we then substitute complex lexical meanings for the free variables corresponding to the lexical entries to obtain a representation of the logical meaning of the sentence,
    \item finally, we use standard theorem proving tools (in first- or higher-order logic) the compute entailment relations between (readings of) sentences.
\end{enumerate}

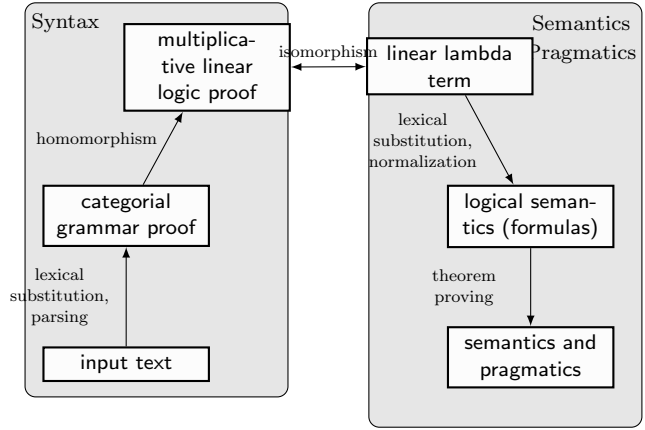
\begin{figure}
\begin{center}
\begin{tikzpicture}[scale=.9,event/.style={rectangle,thick,draw,fill=gray!02,text width=2cm,
		text
                centered,font=\sffamily,anchor=north},label/.style={scale=.8,text
                width=2cm, text
                centered},smalllabel/.style={scale=.55,text
                width=1.5cm,text centered}]
\draw [rounded corners,fill=gray!20] (-5.0em,-9.5em) rectangle (8.0em,10.0em);
\node (syn) at (-3.0em,9.0em) {Syntax};
\draw [rounded corners,fill=gray!20] (12.0em,-11.0em) rectangle (25.5em,10.0em);
\node (sem) at (22.5em,9.0em) {Semantics};
\node (prag) at (22.5em,7.5em) {Pragmatics};
%\draw[line width=0.3cm,color=green!30,cap=round,join=round] (0.5em,-5.5em) -- (17.3em,3em);
%\node [rounded corners,fill=blue!30,draw=black] (lexicon) at
%(10.0em,-0.5em) {\ lexicon\ };
\node [event] (text) at (0em,-7em) {input text};
\node [event] (cg) at (0em,1em) {categorial grammar proof};
\node [event] (mll) at (4em,9em) {mul\-ti\-pli\-ca\-tive linear logic proof};
\node [event] (ll) at (16em,8.3em) {linear lambda term};
\node [event] (sem) at (20em,1em) {logical semantics (formulas)};
\node [event] (prag) at (20em,-6em) {semantics and pragmatics};
\path[>=latex,->] (cg) edge node[left] [label] {{\small homomorphism\qquad
    \qquad \ \ \ \ \ \ \ \qquad \qquad }\ } (mll);
\path[>=latex,<->] (mll) edge node[above] [label] {\small isomorphism} (ll);
\path[>=latex,->] (ll) edge node[left] [label] {\small lexical substitution, normalization} (sem);
\path[>=latex,->] (text) edge node[left] [label] {\small lexical
  substitution, parsing} (cg);
\path[>=latex,->] (sem) edge node[left] [label] {\small theorem
  proving} (prag);
%\node [rounded corners,fill=blue!30,draw=black,text width=2cm,text centered] (ana) at
%(10.0em,-1.2em) {anaphora resolution};
%\path[>=latex,->] (cg) edge node[above] {} %[smalllabel] {c-command}
% (ana);
%\path[>=latex,->] (sem) edge node[above] {} %[smalllabel] {accessibility constraints}
% (ana);
%\path[>=latex,->] (prag) edge node[right] {} %[smalllabel] {
                                %right-frontier constraint}
%(ana);
\end{tikzpicture}
\end{center}
\caption{The standard architecture of type-logical grammars}
\label{fig:arch}
\end{figure}

\begin{table}
    \centering
    \begin{tabular}{ccc}
         \infer[\bo\backslash E\bc]{\Gamma,\Delta\vdash B}{\Gamma\vdash A & \Delta\vdash A\backslash B} && 
         \infer[\bo/ E\bc]{\Gamma,\Delta\vdash B}{\Gamma\vdash B/A & \Delta\vdash A} \\ \\
         \infer[\bo\backslash I\bc]{\Gamma\vdash A\backslash B}{A,\Gamma\vdash B} &&
         \infer[\bo/I\bc]{\Gamma,B/A}{\Gamma,A\vdash B}
    \end{tabular}
    \caption{The Lambek calculus}
    \label{tab:lambek}
\end{table}

\begin{table}
    \centering
    \begin{tabular}{ccc}
         \infer[\bo\multimap E\bc]{\Gamma,\Delta\vdash (M\, N):B}{\Gamma\vdash N:A & \Delta\vdash M:A\multimap B} && 
         \infer[\bo\multimap I\bc]{\Gamma\vdash \lambda x.M:A\backslash B}{x:A,\Gamma\vdash M:B} 
    \end{tabular}
    \caption{The multiplicative intuitionistic linear logic with linear lambda term labeling}
    \label{tab:mill}
\end{table}

To make this more concrete, we present a very simple example, using the Lambek calculus. The Lambek calculus has two connectives\footnote{We ignore the product connectives `$\bullet$' here, since it has somewhat more complicated natural deduction rules and it is not used in the examples.}, $A/B$, pronounced $A$ \emph{over} $B$, representing a formula looking for a $B$ constituent to its right to form an $A$, and $B\backslash A$,  pronounce $B$ \emph{under} $A$, representing a formula looking for a $B$ constituent to its left to form an $A$. Table~\ref{tab:lambek} shows the logical rules of the calculus. 
We'll look at the French sentence `Un Su\'{e}dois a gagn\'{e} un prix Nobel' (\emph{A Swede won a Nobel prize}). 
Figure~\ref{fig:exa} shows a Lambek calculus proof of this sentence. It shows that when we assign the formula $n$, for (common) noun, to `prix' and $n\backslash n$ to `Nobel', we can derive `prix Nobel' as an $n$. Similarly, when we assign $np/n$ to `un' we can combine this with `prix Nobel' of type $n$ to produce `un prix Nobel' as a noun phrase $np$. We can continue the proof as shown in Figure~\ref{fig:exa} to show that `Un Su\'{e}dois a gagn\'{e} un prix Nobel' is a main, declarative sentence $s$.

\renewcommand{\circ}{,}
\begin{figure}
\[
\scalebox{0.6666}{\infer[\bo\backslash\textit{E}\bc]{\textrm{Un}\circ_{}\textrm{suédois}\circ_{}\textrm{a}\circ_{} \textrm{gagné}\circ_{} \textrm{un}\circ_{} \textrm{prix}\circ_{}\textrm{Nobel}\rule[-.2ex]{0pt}{.9em} \vdash s_{\textit{main}}\rule[-.2ex]{0pt}{.9em}
}{     \infer[\bo/\textit{E}\bc]{\textrm{Un}\circ_{}\textrm{suédois}\rule[-.2ex]{0pt}{.9em} \vdash np\rule[-.2ex]{0pt}{.9em}
     }{\infer[\bo\textit{Lex}\bc]{np /_{}n\rule[-.2ex]{0pt}{.9em}
          }{\textrm{Un}\rule[-.2ex]{0pt}{.9em}}
          &
\infer[\bo\textit{Lex}\bc]{n\rule[-.2ex]{0pt}{.9em}
          }{\textrm{suédois}\rule[-.2ex]{0pt}{.9em}}
     }
     &
     \infer[\bo/\textit{E}\bc]{\textrm{a}\circ_{} \textrm{gagné}\circ_{} \textrm{un}\circ_{} \textrm{prix}\circ_{}\textrm{Nobel}\rule[-.2ex]{0pt}{.9em} \vdash np \bs_{}s_{\textit{main}}\rule[-.2ex]{0pt}{.9em}
     }{\infer[\bo\textit{Lex}\bc]{(np \bs_{}s_{\textit{main}}) /_{}(np \bs_{}s_{\textit{ppart}})\rule[-.2ex]{0pt}{.9em}
          }{\textrm{a}\rule[-.2ex]{0pt}{.9em}}
          &
          \infer[\bo/\textit{E}\bc]{\textrm{gagné}\circ_{} \textrm{un}\circ_{}(\textrm{prix}\circ_{}\textrm{Nobel} \rule[-.2ex]{0pt}{.9em} \vdash np \bs_{}s_{\textit{ppart}}\rule[-.2ex]{0pt}{.9em}
          }{\infer[\bo\textit{Lex}\bc]{(np \bs_{}s_{\textit{ppart}}) /_{}np\rule[-.2ex]{0pt}{.9em}
               }{\textrm{gagné}\rule[-.2ex]{0pt}{.9em}}
               &
               \infer[\bo/\textit{E}\bc]{\textrm{un}\circ_{} \textrm{prix}\circ_{}\textrm{Nobel} \rule[-.2ex]{0pt}{.9em} \vdash np\rule[-.2ex]{0pt}{.9em}
               }{\infer[\bo\textit{Lex}\bc]{np /_{}n\rule[-.2ex]{0pt}{.9em}
                    }{\textrm{un}\rule[-.2ex]{0pt}{.9em}}
                    &
                    \infer[\bo\backslash\textit{E}\bc]{\textrm{prix}\circ_{}\textrm{Nobel}\rule[-.2ex]{0pt}{.9em} \vdash n\rule[-.2ex]{0pt}{.9em}
                    }{\infer[\bo\textit{Lex}\bc]{n\rule[-.2ex]{0pt}{.9em}
                         }{\textrm{prix}\rule[-.2ex]{0pt}{.9em}}
                         &
\infer[\bo\textit{Lex}\bc]{n \bs_{}n\rule[-.2ex]{0pt}{.9em}
                         }{\textrm{Nobel}\rule[-.2ex]{0pt}{.9em}}
                    }
               }
          }
     }
}}
\]
\caption{Lambek calculus proof of  `Un Su\'{e}dois a gagn\'{e} un prix Nobel' (\emph{A Swede won a Nobel prize}).}
\label{fig:exa}
\end{figure}

The lambda-term of the corresponding linear logic proof is $(g (u\, p))(u\, s)$ (we have simplified a bit here, treating `a gagn\'{e}' and `prix Nobel' as units).  We then substitute the lexical semantics to obtain the logical representation of the meaning of this sentence. The simple substitutions are \emph{su\'
{e}dois} for $s$ and \emph{prix\_Nobel} for $p$. The two complicated substitutions are the two occurrences of $u$ which are translated as follows.
\[
\lambda P^{e\rightarrow t}\lambda Q^{e\rightarrow t} \exists x. [(P\, x) \wedge (Q,x)]
\]
This is the standard Montague-style analysis of a generalised quantifier. It abstracts over two properties $P$ and $Q$ and states that there is an $x$ which satisfies these two properties. Because of our choice of $np$ for the quantifier (instead of a more standard higher-order type like $s/(np\backslash s)$, the type for the transitive verb has to take care of the quantifier scope. The lexical entry for the transitive verb below chooses the subject wide scope reading.
\[
\lambda N^{(e\rightarrow t)\rightarrow t}\lambda M^{(e\rightarrow t)\rightarrow t} (M\,\lambda x. (N\, \lambda y. \textit{gagner}(x,y)))
\]
Substituting these terms in into the lambda term for the derivation and normalising the resulting term produces the following.
\[
\exists x \exists y. [ \textit{su\'{e}dois}(x) \wedge \textit{prix}\_\textit{Nobel}(y) \wedge \textit{gagner}(x,y)]
\]

%The reader may be tempted to think that the tricky part is actually putting the right formulas in the lexicon. 

Even though this is an admittedly simple example, it is important to note that, although slightly simplified for presentation here, the output for this example and other examples in this paper are automatically produced by the wide-coverage French parser which is part of the Grail family of theorem provers \cite{Moot2017}: Grail uses a deep learning model to predict the correct formulas for each word, finds the best way to combine these lexical entries and finally produces a representation of a logical formula. The full Grail output for the meaning of the example sentence is shown in Figure~\ref{fig:semsuedois}. Grail uses discourse representation structures \cite{drtbook} for its meaning representation, which is essentially a graphical way to represent formulas in first-order logic. Besides providing a readable presentation of formulas, discourse representation structures also provide a dynamic way of binding, with applications to the treatment of anaphora in natural language.

\begin{figure}
\begin{center}
\includegraphics{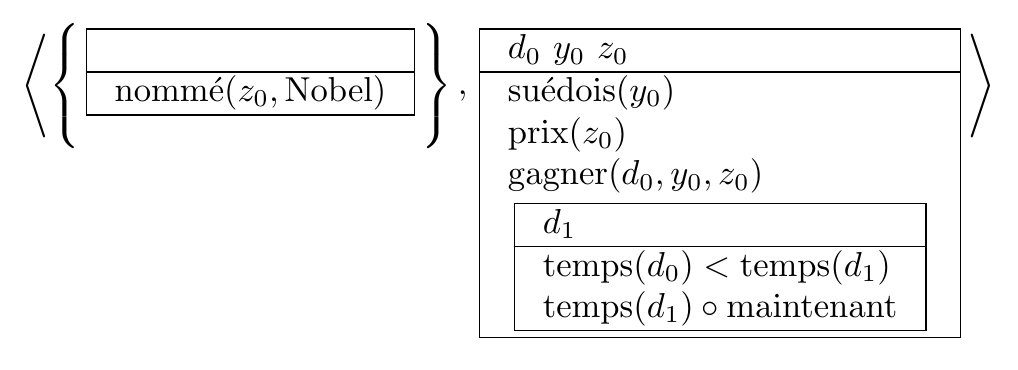}
\end{center}
    \caption{Grail output for the semantics of `Un Su\'{e}dois a gagn\'{e} un prix Nobel'}
    \label{fig:semsuedois}
\end{figure}
The variables $d_0$, $y_0$ and $z_0$ in the top part of the rightmost box represent existentially quantified variables, $y_0$ is a swede, $z_0$ is a prize (named after Nobel) and $d_0$ is a variable for an eventuality   --- essentially denoting a slice of space-time, the inner box indicates that this `winning' even must have occurred at a time before `maintenant' (\emph{now}).

Even though the meaning assigned is in some ways simplistic, the advantage is that it can be automatically obtained and that it is of exactly the right form for logic-based entailment tasks.

%\newpage
\section{Textual Entailment}

What can dialogical argumentation contribute to the study of textual entailment? 
In natural language processing, \emph{textual entailment} is usually defined  as a relation between text  fragments  that holds whenever the truth of one text fragments follows from another text. \emph{Textual entailment recognition} is the task of determining, given text fragments, whether the relation of textual entailment holds between these texts. 

Our examples below are taken form the FraCaS benchmark, but translated into French; This is due to the fact that our methodology involves the use of Grail and the latter is developed mainly for the French language.
The FraCaS benchmark was built in the mid 1990s; the aim was developing a general framework from computational semantics. The data set consists of problems each containing one or more statements and one yes/no-question. 
An example taken from the date set is the following 

\begin{exe}

\ex A Swede won a Nobel prize.
 \ex Every Swede is a Scandinavian.

\ex Did a Scandinavian win a Nobel prize?
[Yes]

\end{exe}

\subsection{First Example}

We illustrate our methodology to solve inference problem using examples. First of all we turn the  question (3) into an assertion i.e. 

\begin{exe}
\ex Some Scandinavian won a Nobel prize. 
\end{exe}

We then translate each sentence in french and  use Grail on each sentence in order to get a logical formula. 
In the enumeration below we report, in order: the sentence in English. A word-for-word translation, then a translation that takes into account the grammar and idiomatic of French and, finally, the logical formula that Grail outputs from the input of the latter

\begin{exe}
\ex
\gll A swede won a nobel prize \\
    Un suédois {a gagné} un nobel prix \\
    \trans Un suédois a gagné le prix Nobel 
    \trans $\exists x \exists y. [ \textit{su\'{e}dois}(x) \wedge prix\_Nobel(y) \wedge gagner(x,y)]$
    \ex 
    \gll Every swede is a scandinavian \\
    Tout suédois est un scandinave  \\
    \trans Tout suédois est scandinave
    \trans $\forall u.(\textit{suédois}(u)\Rightarrow \textit{scandinave}(u)) $
    \ex 
    \gll Some Scandinavian won a Nobel prize \\
     Un scandinave {à gagné} un Nobel prix \\
     \trans Un scandinave à gagné un prix Nobel
     \trans $\exists w.\exists z. [ (\textit{scandinave}(w)\wedge (prix\_Nobel(z)\wedge gagner(w,z))]$
    \end{exe}

 We then construct a \emph{winning strategy} for the formula $H_1\wedge H_2\ldots\wedge  \ H_n\Rightarrow C$ where each $H_i$ is the logical formula that Grail associates to each statement from the data set, and $C$ is the formula that Grail associates to the assertion obtained from the pair question-answer in the data-set. 
 
 \medskip
 \medskip
\footnotesize{ \[F= (\exists x \exists y. [ \textit{su}(x) \wedge p\_N(y) \wedge g(x,y)] \wedge \forall u. (su(u)\rightarrow sc(u)))\Rightarrow \exists w.\exists z. [ (\textit{sc}(w)\wedge (p\_N(z)\wedge g(w,z))]  \]}

\normalsize
In the above formula \emph{su} stands for \emph{suédois}, \emph{p\_N} for \emph{prix\_Nobel}, \emph{g} pour \emph{gagner} and \emph{sc} for \emph{scandinave}. A winning strategy for the formula $F$ is showed thereafter in two steps.

  \begin{center}
\begin{tikzpicture}[scale=0.8,->,level/.style={sibling distance = 2cm,
  level distance = 0.7cm}] 
\node{\footnotesize{$!,(\exists x \exists y. [ \textit{su}(x) \wedge p\_N(y) \wedge g(x,y)] \wedge \forall u. (su(u)\rightarrow sc(u)))\Rightarrow \exists w.\exists z. [ (\textit{sc}(w)\wedge (p\_N(z)\wedge g(w,z))]$}}
child{node(1){\footnotesize{$?,(\exists x \exists y. [ \textit{su}(x) \wedge p\_N(y) \wedge g(x,y)] \wedge \forall u. (su(u)\rightarrow sc(u)))$}}
child{node(2){\footnotesize{$?,\wedge_1$}}
child{node(3){\footnotesize{$!,(\exists x \exists y. [ \textit{su}(x) \wedge p\_N(y) \wedge g(x,y)]$}}
child{node(4){\footnotesize{$?,\exists$}}
child{node(5){\footnotesize{$!, \exists y. [ \textit{su}(x_0) \wedge p\_N(y) \wedge g(x_0,y)] $}}
child{node(6){\footnotesize{$?,\exists$}}
child{node(7){\footnotesize{$!, \textit{su}(x_0) \wedge p\_N(x_1) \wedge g(x_0,x_1)]$}}
child{node(8){\footnotesize{$?,\wedge_1$}}
child{node(9){\footnotesize{$!,su(x_0)$}}
child{node(10){\footnotesize{$?,\wedge_2$}}
child{node(11){\footnotesize{$!,N(x_1)\wedge g(x_0,x_1)$}}
child{node(12){\footnotesize{$?,\wedge_1$}}
child{node(13){\footnotesize{$!,N(x_1)$}}
child{node(14){\footnotesize{$?,\wedge_2$}}
child{node(15){\footnotesize{$!,g(x_0,x_1)$}}
child{node(16){\footnotesize{$?,\wedge_2$}}
child{node(17){\footnotesize{$!,\forall u. (su(u)\Rightarrow sc(u))$}}
child{node(18){$?,\forall[x_o/u]$}
child{node(19){\footnotesize{$!,su(x_0)\Rightarrow sc(x_0) $}}
child{node(20){\footnotesize{$?,su(x_0)$}}
child{node(21){\footnotesize{$!,sc(x_0)$}}
child{node(22){\footnotesize{$\exists w \exists z( (sc(w)\wedge(N(z) \wedge g(w,z)))$}}}}}}}}}}}}}}}}}}}}}}}};

 \draw[thick,blue,dotted] (2) to[out=-180,in=180](1); 
      \draw[thick,blue,dotted] (4) to[out=180,in=180](3);
      \draw[thick,blue,dotted] (6) to[out=180,in=180](5);
      
      \draw[thick,blue,dotted] (8) to[out=180,in=180](7); 
         \draw[thick,blue,dotted] (10) to[out=180,in=180](7); 
       \draw[thick,blue,dotted] (12) to[out=180,in=180](11); 
        \draw[thick,blue,dotted] (14) to[out=180,in=180](11);
         \draw[thick,blue,dotted] (16) to[out=0,in=0](1);
          \draw[thick,blue,dotted] (18) to[out=180,in=180](17);
  \draw[thick,blue,dotted] (20) to[out=180,in=180](19);
  \draw[thick,blue,dotted] (22) to[out=0,in=0](1);
         \end{tikzpicture}

          \begin{tikzpicture}[scale=1.2,->,level/.style={sibling distance = 2cm,
  level distance = 0.7cm}] 
\node(1){\footnotesize{$?,\exists$}}
child{node(2){\footnotesize{$\exists z(sc(x_0)\wedge (N(z)\wedge g(x_0,z)$}}
child{node(3){\footnotesize{$?,\exists$}}
child{node(4){\footnotesize{$!,sc(x_0)\wedge (N(x_1)\wedge g(x_0,x_1))$}}
child{node(5){\footnotesize{$?,\wedge_1$}} child{node(6){\footnotesize{$!,sc(x_0)$}}}}
child{node(7){\footnotesize{$?,\wedge_2$}} child{node(8){\footnotesize{$!,N(x_1)\wedge g(x_0,x_1)$}}
child{node(9){\footnotesize{$?,\wedge_1$}} child{node(10){\footnotesize{$N(x_1)$}}}}
child{node(11){\footnotesize{$?,\wedge_2$}}
child{node(12){\footnotesize{$!,g(x_0,x_1)$}}}}}}}}};
        \draw[thick,blue,dotted] (2) to[out=-180,in=180](1); 
      \draw[thick,blue,dotted] (4) to[out=180,in=180](3);
      \draw[thick,blue,dotted] (6) to[out=180,in=180](5);
      
      \draw[thick,blue,dotted] (8) to[out=0,in=0](7); 
         \draw[thick,blue,dotted] (10) to[out=180,in=180](9); 
       \draw[thick,blue,dotted] (12) to[out=0,in=0](11); 
  \end{tikzpicture}
\end{center}

\subsection{Second Example}

\begin{exe}
\ex Some Irish delegates finished the survey on time.
\ex Did any delegates finish the survey on time?
[Yes]
\end{exe}
The answer to the question is affirmative. This means that if (8) is true then the sentence \textit{``some delegate finished the survey on time''} must also be true.
\begin{exe}
\ex
\gll  Some Irish delegates finished the survey on time \\
    Certains irlandais délégués  {ont terminé} {l'} enquête à temps      \\
    \trans Certain délegués irlandais ont términé l'enquête à temps
    \trans $\exists x \exists y ((\textit{délegué}(x)\wedge\textit{irlandais}(x)) \wedge ( \textit{enquête}(y) \wedge \textit{terminé-à-temps}(x,y))) $
    \ex 
    \gll Some delegates finished the survey on time  \\
    Certains délégues {ont terminé} {l'} enquête à temps\\ 
    
    \trans Certains délégues ont términé l'enquête à temps 
    \trans  $\exists x \exists y (\textit{délegué}(x) \wedge ( \textit{enquête}(y) \wedge \textit{terminé-à-temps}(x,y))) $
   
\end{exe}

We have that $F_1\Rightarrow F_2$. Where \[F_1=\exists x \exists y ((\textit{délegué}(x)\wedge\textit{irlandais}(x)) \wedge ( \textit{enquête}(y) \wedge \textit{terminé-à-temps}(x,y)))\]
\[F_2=\exists x \exists y (\textit{délegué}(x) \wedge ( \textit{enquête}(y) \wedge \textit{terminé-à-temps}(x,y)))\]

\begin{center}
 \quad 
\begin{tikzpicture}[scale=1.2,->,level/.style={sibling distance = 2cm,
  level distance = 0.6cm}] 
  
  \node{\footnotesize{$\exists x\exists y(D(x)\wedge I(x))\wedge (E(y)\wedge TaT(x,y)))\rightarrow \exists w \exists z (D(w)\wedge (E(z)\wedge Tat(w,z))$}}
child{node(1){\footnotesize{$\exists x\exists y(D(x)\wedge I(x))\wedge (E(y)\wedge TaT(x,y)))$}}
child{node(2){\footnotesize{$?,\exists$}}
child{node(3){\footnotesize{$\exists y ((D(x_0)\wedge I(x_0))\wedge (E(y)\wedge Tat(x_0,y))$}}
child{node(4){\footnotesize{$?,\exists$}}
child{node(5){\footnotesize{$(D(x_0)\wedge I(x_0))\wedge (E(x_1)\wedge Tat(x_o,x_1))$}}
child{node(6){\footnotesize{$?,\wedge_1$}}
child{node(7){\footnotesize{$D(x_0)\wedge I(x_0)$}}
child{node(8){\footnotesize{$?,\wedge_1$}}
child{node(9){\footnotesize{$!,D(x_0)$}}
child{node(10){\footnotesize{$?,\wedge_2$}}
child{node(11){\footnotesize{$!,E(x_1)\wedge Tat(x_o,x_1)$}}
child{node(12){\footnotesize{$?\wedge_1$}}
child{node(13){\footnotesize{$!,E(x_1)$}}
child{node(14){\footnotesize{$?,\wedge_2$}}
child{node(15){\footnotesize{$!,Tat(x_0,x1)$}}
child{node(16){\footnotesize{$\exists w \exists z (D(w)\wedge (E(z)\wedge Tat(w,z))$}}
child{node(17){\footnotesize{$?,\exists$}}
child{node(18){\footnotesize{$ \exists z (D(x_0)\wedge (E(z)\wedge Tat(x_0,z))$}}
child{node(19){\footnotesize{$?,\exists$}}
child{node(20){\footnotesize{$ (D(x_0)\wedge (E(x_1)\wedge Tat(x_0,x_1))$}}
}}}}}}}}}}}}}}}}}}}}

child{node(21){\footnotesize{$?,\wedge_1$}}
child{node(22){\footnotesize{$D(x_0)$}}}}
child{node(23){\footnotesize{$?,\wedge_2$}} child{node(24){\footnotesize{$E(x_1)\wedge Tat(x_0,x_1)$}} child{node(25){\footnotesize{$?,\wedge_1$}}
child{node(26){\footnotesize{$!,E(x_1)$}}}}  child{node(27){\footnotesize{$?,\wedge_2$}}
child{node(28){\footnotesize{$!,Tat(x_0,x_1)$}}
}}}}

;
 \draw[thick,blue,dotted] (2) to[out=-180,in=180](1); 
      \draw[thick,blue,dotted] (4) to[out=-180,in=180](3);
      \draw[thick,blue,dotted] (6) to[out=0,in=0](5);
      
      \draw[thick,blue,dotted] (8) to[out=0,in=0](7); 
         \draw[thick,blue,dotted] (10) to[out=180,in=180](5); 
       \draw[thick,blue,dotted] (12) to[out=0,in=0](11); 
        \draw[thick,blue,dotted] (14) to[out=0,in=0](11);
         \draw[thick,blue,dotted] (16) to[out=0,in=0](1);
          \draw[thick,blue,dotted] (18) to[out=0,in=0](17);
  \draw[thick,blue,dotted] (20) to[out=180,in=180](19);
  \draw[thick,blue,dotted] (22) to[out=0,in=0](21);
   \draw[thick,blue,dotted] (24) to[out=0,in=0](23);
    \draw[thick,blue,dotted] (26) to[out=0,in=0](25);
     \draw[thick,blue,dotted] (28) to[out=0,in=0](27);

;
\end{tikzpicture}
\end{center}

\subsection{Third Example}

\begin{exe}

\ex No delegate finished the report on time 
 \ex Did any Scandinavian delegate finished the report on time? [No]

\end{exe}

In this example the answer get a negative reply.  A positive answer would be implied by the existence of a Scandinavian delegate who finished the report in the time allotted. Thus the sentence (8) plus the sentence \textit{Some Scandinavian delegate finished the report on time} should imply a contradiction. We first translate the two sentences in French and use Grail to get the corresponding logical formulas.

\begin{exe}
\ex
\gll  No delegate finished the report on time\\
   Aucun délégué {n'a terminé} le  rapport à temps \\
    \trans Aucan délégué n'a terminé le rapport à temps 
    \trans $\forall x (\textit{délégué}(x)\Rightarrow\neg \textit{terminé-le rapport-à-temps}(x))$
    \ex 
    \gll Some Scandinavian delegate finished the report on time\\
     Un  scandinave délegué  {a terminé} le rapport à temps\\
    
    \trans un délégué scandinave a terminé le rapport à temps 
    \trans $\exists x (  ( \textit{délégue}(x) \wedge \textit{scandinave}(x)) \wedge\textit{ términe-le-rapport-a-temps}(x))$ 
   
\end{exe}

    The two formulas
    
    \[F_1=\forall x (\textit{délégué}(x)\Rightarrow\neg \textit{terminé-le rapport-à-temps}(x))\]
    \[F_2= \exists x (  ( \textit{délégue}(x) \wedge \textit{scandinave}(x)) \wedge\textit{ términe-le-rapport-a-temps}(x))\]
    
    Are contradictory. So it exists a winning strategy for the formula $\neg(F_1\wedge F_2)$
    
    \medskip

    \begin{center}
 \quad 
\begin{tikzpicture}[scale=1.5,->,level/.style={sibling distance = 2cm,
  level distance = 0.7cm}] 
\node{\footnotesize{$\neg((\forall x (D(x)\Rightarrow \neg Trt(x))\wedge (\exists y((D(y)\wedge Sc(y))\wedge Trt(y)))$}}
child{node(1){\footnotesize{$(\forall x (D(x)\Rightarrow \neg Trt(x))\wedge (\exists y((D(y)\wedge Sc(y))\wedge Trt(y))$}}
child{node(2){\footnotesize{$?,\wedge_2$}}
child{node(3){\footnotesize{$!,\exists y((D(y)\wedge Sc(y))\wedge Trt(y))$}}
child{node(4){\footnotesize{$?,\exists$}}
child{node(5){\footnotesize{$(D(w)\wedge Sc(w))\wedge Trt(w)$}}
child{node(6){\footnotesize{$?,\wedge_1$}}
child{node(7){\footnotesize{$!,\forall x(D(x)\Rightarrow \neg Trt(x))$}}
child{node(8){\footnotesize{$?,\forall[w/x]$}}
child{node(9){\footnotesize{$!,D(w)\Rightarrow \neg Trt(w) $}}
child{node(10){\footnotesize{$?,\wedge_1$}}
child{node(11){\footnotesize{$!,D(w)\wedge Sc(w)$}}
child{node(12){\footnotesize{$?,\wedge_1$}}
child{node(13){\footnotesize{$!,D(w)$}}
child{node(14){\footnotesize{$?,\wedge_2$}}
child{node(15){\footnotesize{$!,Trt(w)$}}
child{node(16){\footnotesize{$?,D(w)$}}
child{node(17){\footnotesize{$!,\neg Trt(w)$}}
child{node(18){\footnotesize{$?,Trt(w)$}}
}}}}}}}}}}}}}}}}}} ;
  \draw[thick,blue,dotted] (2) to[out=-180,in=180](1); 
      \draw[thick,blue,dotted] (4) to[out=-180,in=180](3);
      \draw[thick,blue,dotted] (6) to[out=0,in=0](1);
      \draw[thick,blue,dotted] (8) to[out=0,in=0](7); 
         \draw[thick,blue,dotted] (10) to[out=180,in=180](5); 
       \draw[thick,blue,dotted] (12) to[out=0,in=0](11); 
        \draw[thick,blue,dotted] (14) to[out=0,in=0](5);
         \draw[thick,blue,dotted] (16) to[out=0,in=0](7);
          \draw[thick,blue,dotted] (18) to[out=0,in=0](17);

\end{tikzpicture}
\end{center}

\subsection{Fourth Example}

In the last example we focus on a series of sentences that our system should not solve, because the question asked neither has a positive nor a negative answer.

\begin{exe}

\ex A Scandinavian won a Nobel prize.

\ex Every Swede is a Scandinavian
\ex Did a Swede win a Nobel prize? [Don't know]

\end{exe}

This means that, on the basis of the information in our possession, we can neither say that a Swede has won a Nobel Prize nor that there are no Swedes who have won a Nobel Prize. 

\begin{exe}
\ex 
\gll A Scandinavian won a Nobel prize\\ 
 Un scandinave {a gagné} un Nobel prix \\
 
 \trans Un scandinave à gagne un prix Nobel
 \trans $\exists x \exists y (scandivane(x)\wedge (prix\_Nobel(y) \wedge gagne(x,y)))$

 \ex 
    \gll Every swede is a scandinavian \\
    Tout suédois est un scandinave  \\
    \trans Tout suédois est scandinave
    \trans $\forall u.(\textit{suédois}(v)\Rightarrow \textit{scandinave}(v)) $

\end{exe}

Call the formula in (19) $F_1$ and the formula in (20) $F_2$. In dialogical logic terms the fact that we do not have enough information neither to answer in a positive fashion nor in a negative fashion to the question (16), means that there is no winning strategy neither for the formula  $F_1\wedge F_2\Rightarrow F_3$ nor for the formula $F_1\wedge F_2\Rightarrow \neg F_3$ where the formula $F_3$ is 

\[F_3= \exists w \exists z (\textit{suédois}(w)\wedge (prix\_Nobel(z)\wedge gagne (w,z)))\]

In general given a sentence $F$ of first order logic there it is not \emph{decidable} whether $F$ is valid. However \emph{in some cases} we can manage this problem. Luckily the present case is one of those we can manage.  
 We consider how a winning strategy for the formula $F_1\wedge F_2\rightarrow F_3$ must look like.
A winning strategy $S$ for this formula will necessarily contain a dialog whose last move $M_n$ is a $\mathbf{P}$-move that asserts  $\textit{suédois}(t)$ for some term $t$ in the language. Since $\textit{suédois}(t)$ is an atomic formula by proposition 3 above $\textit{suédois}(t)$ must occur both as a positive and negative gentzen sub-formula of $F_1\wedge F_2\Rightarrow F_3$ but this is not the case. Thus there is no winning strategy for the latter formula.\\
Let us now discuss why there is no winning strategy for the formula $F_1\wedge F_2\Rightarrow \neg F_3$. First of all proposition 1 assures us that each game won by $\mathbf{P}$ ends by the assertion of some atomic formula and that such assertion is a $\mathbf{P}$-move. By proposition 3 above the only candidate for this is again $\textit{suédois}(t)$ for all term $t$ in the language. If the move $\mathbf{P}$ $M_k$ is an assertion of $\textit{suéduois}(t)$ then it must be an attack. Suppose is not the case then it should be a defence. This means that that there is a formula $F$  of the form $\forall w \textit{suédois}(w)$ or $\exists w (su\acute{e}dois(w)) $ or $F'\vee su\acute{e}dois(w)$ or $F'\wedge \textit{suédois}(w)$ or $F'\Rightarrow su\acute{e}dois(t)$ that $\mathbf{P}$ asserts $F$. This imply that such formula $F$ must be positive Gentzen-subformula of $F_1\wedge F_2\Rightarrow \neg F_3$. But not suchformula exists. Thus the move $M_n$ asserting $su\acute{e}dois(t)$ must be an attack. Since the only formula that can be attacked by this means is the is the formula $su\acute{e}dois(t)\Rightarrow scandinave(t)$.  O can answer back by asserting $scandinave(t)$ thus $\mathbf{P}$ cannot win the game. Thus there is no winning strategy for the formula $F_1\wedge F_2\Rightarrow \neg F_3$

\section{Conclusion}
 
In this paper, we adapted our simple version of argumentative dialogues and strategies of \cite{CPR2017jla} to two-sided sequents (hypotheses and conclusions): this point of view  better matches natural language statements, because  the assumptions sentences of a textual entailment task can be viewed as sequent calculus hypotheses, while the text conclusion can be viewed as the conclusion of the sequent.

In the present paper, we successfully use the syntactic and semantic platform Grail to ``translate'' natural language sentences into DRS that can be viewed as logical formulas. 

This lead us closer to inferentialist semantics: a sentence $S$ can be interpreted as all argumentative dialogues  in natural language whose conclusion is $S$ --- under assumptions corresponding to word meaning and to the speaker beliefs. 

We are presently working to extend our work with semantics taking place in classical first order logic to a broader setting in which semantics takes place in modal logic.  Indeed,   modal reasoning is rather common in natural language argumentation. 

Regarding the architecture of our model of natural language argumentation we would like to encompass lexical meaning as axioms along the lines of \cite{CPR2017jla} and to use hypotheses to model the way the two speakers differ in their expectations,  beliefs and knowledge, taking insights from existing works on functional roles in dialogue modelling.

\bibliographystyle{spmpsci}
\bibliography{ref}

\end{document}